\documentclass[12pt]{article}
\usepackage{amsmath}
\usepackage{graphicx}
\usepackage{natbib}
\usepackage{url} 

\newcommand{\blind}{0}

\usepackage{verbatim,epsfig,amsthm}
\usepackage{amssymb,rotating}
\usepackage[nodisplayskipstretch]{setspace}
\usepackage{amsfonts}
\usepackage{multirow}
\usepackage{epsfig,latexsym,enumerate,paralist}
\usepackage{mathtools,color,etoolbox,hyperref,float,mathrsfs}
\usepackage{caption,subcaption,threeparttable}

\usepackage{dsfont}

 \usepackage{adjustbox,enumitem, paralist}

\usepackage{stackengine}

\newcommand{\bd}{\boldsymbol}
\newcommand{\mb}{\mathbf}

\newcommand{\be}{\begin{equation}}
\newcommand{\ee}{\end{equation}}

\DeclareMathOperator*{\argmax}{arg\,max}

\let\Re\undefined
\let\Im\undefined
\DeclareMathOperator{\Re}{\mathfrak{R}}
\DeclareMathOperator{\Im}{\mathfrak{I}}

\DeclareMathOperator{\diag}{diag}

\DeclareMathOperator{\rank}{rank}
\DeclareMathOperator{\cov}{cov}
\DeclareMathOperator{\corr}{corr}
\DeclareMathOperator{\var}{var}

\DeclareMathOperator{\tr}{tr}
\DeclareMathOperator{\lspan}{span}

\DeclareMathOperator{\PVE}{PVE}

\newtheorem{theorem}{Theorem}

\newtheorem{remark}{Remark}

\newtheorem{prop}{Proposition}

\usepackage[mathscr]{euscript}

\begin{document}

\def\spacingset#1{\renewcommand{\baselinestretch}%
{#1}\small\normalsize} \spacingset{1}


\if0\blind
{
  \title{\bf D-CDLF: Decomposition of Common and Distinctive Latent Factors for Multi-view High-dimensional Data}
	\author{Hai Shu 
\hspace{.2cm}\\
\\
		Department of Biostatistics, New York University\\
		\\
		Email:  hs120@nyu.edu
		}
  \maketitle
} \fi

\if1\blind
{
  \bigskip
  \bigskip
  \bigskip
  \begin{center}
    {\LARGE\bf Title}
\end{center}
  \medskip
} \fi

\bigskip
\begin{abstract}
A typical approach to the joint analysis of multiple high-dimensional data views is to decompose each view's data matrix into three parts: a low-rank common-source matrix generated by common latent factors of all data views, a low-rank distinctive-source matrix generated by distinctive latent factors of the corresponding data view, and an additive noise matrix. Existing decomposition methods often focus on the uncorrelatedness between the common latent factors and distinctive latent factors, but inadequately address the equally necessary uncorrelatedness between distinctive latent factors from different data views. We propose a novel decomposition method, called Decomposition of Common and Distinctive Latent Factors (D-CDLF), to effectively achieve both types of uncorrelatedness for two-view data. We also discuss the estimation of the D-CDLF under high-dimensional settings.

\end{abstract}

\noindent%
\vfill

\newpage
\spacingset{1.9}

\section{Introduction}
\label{s:intro}

Let $\bd{y}_{k,i}\in\mathbb{R}^{p_k}$ ($1\le k\le K$, $1\le i\le n$) be the $k$-th data view of the $i$-th subject with $p_k$ observable variables
(e.g., $p_1$ brain nodes in FDG-PET data for the first view, and $p_2$ SNPs in genotyping data for the second view).  
Assume that $\{\bd{y}_{k,i}\}_{i=1}^n$ are $n$ independent and identically distributed (i.i.d.) observations of a random vector $\bd{y}_k$.
A typical model for multi-view high-dimensional data conducts the  decomposition: 
\vspace{-0.1cm}
\be\label{eqn: y=c+d+e}
\bd{y}_k=\bd{x}_k+\bd{e}_k
=\bd{c}_k+\bd{d}_k+\bd{e}_k
=\mb{B}_{k,c}(c^{(1)},\dots,c^{(L_c)})^\top+\mb{B}_{k,d}
(d_k^{(1)},\dots,d_k^{(L_k)})^\top+\bd{e}_k,
\vspace{-0.1cm}
\ee
for $k=1,\dots, K$, where $\bd{x}_k$ is the signal, 
an approximation of $\bd{y}_k$,
assumed to be
generated by a small number of latent factors
to 
avoid the curse of high dimensionality~\citep{Yin88}, $\bd{e}_k$ is the residual noise, 
$\bd{c}_k$ and $\bd{d}_k$ are the common-source 
and distinctive-source parts of $\bd{x}_k$, respectively,
generated by
the {\it common latent factors (CLFs)} $\{c^{(\ell)}\}_{\ell=1}^{L_c}$ of $\{\bd{x}_k\}_{k=1}^K$
and  the {\it distinctive latent factors (DLFs)} $\{d_k^{(\ell)}\}_{\ell=1}^{L_k}$ of $\bd{x}_k$,  
and 
$\{\mb{B}_{k,c},\mb{B}_{k,d}\}$ are coefficient matrices.
As the focus is on data variation, 
all random variables in~\eqref{eqn: y=c+d+e} are assumed to be mean-zero.
{For biomedical data, 
the common and distinctive latent factors
(CDLFs)
 can  
be viewed as the common and distinctive biological mechanisms
underlying multi-view data, 
manifested through their {\it concrete representations}, 
common- and distinctive-source signals $\bd{c}_k$ and $\bd{d}_k$, within the original data domain of the $k$-th data~view.}

Two main issues exist in previous work~\citep{Lofs11,schouteden2013sca,Zhou16, Lock13, Feng18,OCon16,gaynanova2019structural, Shu18,Shu2022d}:
(i)
Insufficient consideration has been given to 
the uncorrelatedness of CDLFs: $\{c^{(\ell)}\}_{\ell=1}^{L_c}\perp \{d_k^{(\ell)}\}_{\ell=1}^{L_k}\perp\{d_{k'}^{(\ell)}\}_{\ell=1}^{L_{k'}}$, $k\ne k'$.
 This property  ensures complete separation of CDLFs.
If a CLF and a DLF are correlated,  
there will be a CLF between them.
For example,
if $\corr(c^{(1)},d_1^{(1)})\ne 0$,
then $d_1^{(1)}$ can be viewed as
a CLF of $c^{(1)}$ and $d_1^{(1)}$,
because
$c^{(1)}=\cov(c^{(1)},d_1^{(1)})d_1^{(1)}/\var(d_1^{(1)})+\epsilon$
with $\epsilon\perp d_1^{(1)}$.
Similarly, two correlated DLFs
from different~data views will have a CLF.
Most methods ~\citep{Lofs11,schouteden2013sca,Zhou16, Lock13, Feng18} 
focus on
the uncorrelatedness between CLFs and DLFs,
but ignore the uncorrelatedness between DLFs from different data views.
 \citet{Shu18,Shu2022d} emphasizes the uncorrelatedness between all DLFs but at the cost of losing the uncorrelatedness between CLFs and DLFs.
Though some  attempts have been made to achieve both types of uncorrelatedness, they either sacrifice some signal
as noise~\citep{OCon16} or offer an asymmetrical decomposition  for identically distributed signals~\citep{gaynanova2019structural}.
(ii) There is a lack of tools that are adaptive to multi-view data to explain the relationship between CDLFs and original variables.

To address the above two issues,
we propose a novel method,  Decomposition of Common and Distinctive Latent Factors (D-CDLF), for $K=2$ data views.
The proposed D-CDLF is the first of its kind to achieve the desirable uncorrelatedness
of CDLFs within and between the two data views.
Additionally, the D-CDLF is accompanied by  two new types of Proportions of Variance Explained (PVEs),
the variable-level PVEs and the view-level PVEs, to measure the joint effects of all CLFs and those of all DLFs on  original variables.

The rest of this paper is organized as follows.
Section~\ref{sec: Preliminaries} introduces some useful notation and the canonical correlation analysis  \citep[CCA;][]{Hote36}
as preliminaries.
Section~\ref{sec: 2-view D-CDLF} proposes the two-view D-CDLF and the variable-level and view-level PVEs.
Section~\ref{sec: estimation} discusses the estimation of 
two-view D-CDLF under high-dimensional settings.
All theoretical proofs are deferred to Section~\ref{sec: proofs}.

\section{Preliminaries}\label{sec: Preliminaries}
\subsection{Notation}
We  introduce some useful notation.
Define $[n]=\{1,\dots, n\}$ for any positive integer $n$.
For a real matrix $\mb{M}=(M_{ij})_{1\le i\le p,1\le j\le n}$, 
the $\ell$-th largest singular value is denoted by $\sigma_\ell(\mb{M})$, and
the $\ell$-th largest eigenvalue when $p=n$
is $\lambda_{\ell}(\mb{M})$.
Denote $\mb{M}^{[s:t,u:v]}$, $\mb{M}^{[s:t,:]}$,  and $\mb{M}^{[:,u:v]}$
as the submatrices 
$(M_{ij})_{s\le i\le t, u\le j\le v}$, $(M_{ij})_{s\le i\le t,1\le j\le n}$, 
and $(M_{ij})_{1\le i\le p, u\le j\le v}$
of $\mb{M}$, respectively.
We write the $j$-th entry of a vector $\bd{v}$ by $\bd{v}^{[j]}$, and $\bd{v}^{[s:t]}=(\bd{v}^{[s]},\bd{v}^{[s+1]},\dots,\bd{v}^{[t]})^\top$.
For matrices $\mb{M}_1,\dots , \mb{M}_N$ of appropriate dimensions,
denote $[\mb{M}_1;\dots;\mb{M}_N]=(\mb{M}_1^\top,\dots,\mb{M}_N^\top)^\top$ to be their row-wise concatenation,
define $[\mb{M}_\ell]_{\ell=a}^b=[\mb{M}_a;\mb{M}_{a+1};\dots; \mb{M}_b]$, which is an empty matrix if $a>b$,
and
define
$[\mb{M}_\ell]_{\ell\in\mathcal{I}}
=[\mb{M}_{i_1};\dots;\mb{M}_{i_p}]$
for an index set $\mathcal{I}=\{i_1,\dots,i_p\}$.
Similarly, define $(\mb{M}_\ell)_{\ell=a}^b$
and $(\mb{M}_\ell)_{\ell\in\mathcal{I}}$
for column-wise concatenations. 
For a set $S=\{s_1,\dots,s_p\}$, 
define $[S]=[s_1;\dots;s_p]$
as the vector form of $S$.
By default, we assume that the elements on the main diagonal of the (rectangular) diagonal matrix in the singular value decomposition (SVD) of a given real matrix are arranged in descending order.

Assume that all random variables are defined on a  probability space $(\Omega, \mathcal{F},P)$.
The $\mathcal{L}^2$ space of all $\mathbb{S}$-valued random variables 
on $(\Omega, \mathcal{F},P)$
 is 
 $\mathcal{L}^2(\Omega, \mathcal{F},P;\mathbb{S})
 =\{
 x:\Omega\to \mathbb{S}\big| \int_\Omega |x(\omega)|^2  P(d\omega)<\infty
\}
 $, with $\mathbb{S}\in \{\mathbb{R},\mathbb{C}\}$.
 For  complex random variables $x$ and $y$,
define the expectation of $x$ by  $E[x]=E[\Re(x)]+\mathbf{i}E[\Im(x)]$,
the covariance of $x$ and $y$ by
$\cov(x,y)=E[(x-E[x])(y-E[y])^*]$, and
their correlation by $\corr(x,y)=\cov(x,y)/\sqrt{\var(x)\var(y)}$
if $\var(x):=\cov(x,x)\ne 0$ and $\var(y)\ne 0$, and otherwise $\corr(x,y)=0$,
where $\Re(x)$ and $\Im(y)$ are the real and imaginary parts of $x$, respectively, and $y^*$ is the complex conjugate of $y$.
By default, 
 we use the inner product of $x$ and $y$ defined by
$\langle x,y \rangle=E[xy^*]$,
and its induced norm of $x$ is $\| x\|=\sqrt{\langle x,x \rangle}$.
With $\langle \cdot,\cdot \rangle$ and $\| \cdot\|$,
the  space
$\mathcal{L}^2(\Omega, \mathcal{F},P;\mathbb{S})$ is a Hilbert space
for $\mathbb{S}\in \{\mathbb{R},\mathbb{C}\}$,
in which the notation $x=y$ means  $P(x=y)=1$ \citep{Shiryaev1996}.

Let $\mathcal{L}_0^2$ be the subspace of all mean-zero real random variables in $\mathcal{L}^2(\Omega, \mathcal{F},P;\mathbb{R})$,
for which the above defined $\langle \cdot,\cdot\rangle$ equals $\cov(\cdot,\cdot )$.
Denote by $(\mathcal{L}_0^2,\cov)$ the inner product space of $\mathcal{L}_0^2$ with $\cov(\cdot, \cdot)$ as its inner product.
The space $(\mathcal{L}_0^2,\cov)$ is also a Hilbert space,
in which $\cos\{\theta(\cdot,\cdot)\}=\corr(\cdot,\cdot)$ and $\|\cdot\|=\sqrt{\var(\cdot)}$.
Note that orthogonality in $(\mathcal{L}_0^2,\cov)$ is equivalent to uncorrelatedness.
Thus, we use the two terms interchangeably in $(\mathcal{L}_0^2,\cov)$.

For a set $\{v_j\}_{j=1}^p$, 
we denote its linear span over $\mathbb{R}$ by
 $\lspan(\{v_j\}_{j=1}^p)=\{\sum_{j=1}^p a_jv_j| a_j\in \mathbb{R}\}$,  and sometimes write it as $\lspan_{\mathbb{R}}(\{v_j\}_{j=1}^p)$ to emphasize $\{a_j\}_{j=1}^p\subseteq \mathbb{R}$.
For a vector $\bd{v}=[v_j]_{j=1}^p$,
write $\lspan(\bd{v}^\top)=\lspan(\{v_j\}_{j=1}^p)$.
For $\bd{x}_k$ ($k=1,2$) in~\eqref{eqn: y=c+d+e} with entries in $(\mathcal{L}_0^2,\cov)$,
define 
$r_c=\rank(\cov(\bd{x}_1,\bd{x}_2))$ and $r_k=\rank(\cov(\bd{x}_k))$.
We have $r_k=\dim(\lspan(\bd{x}_k^\top))$.

\subsection{Canonical correlation analysis}\label{subsec: CCA}
The CCA method \citep{Hote36} sequentially finds the most correlated variables, called {\it canonical variables}, between the two subspaces 
$\{\lspan(\bd{x}_k^\top)\}_{k=1}^2$ in $(\mathcal{L}_0^2,\cov)$. For $1\le \ell\le  r_c$, the $\ell$-th pair of  canonical variables are defined as
\be
\begin{split}\label{CCA def}
	&
	\{z_1^{(\ell)},z_2^{(\ell)}\}	\in \argmax_{\{z_k\}_{k=1}^2}\ \corr(z_1,z_2)\quad \text{subject to}\\
	&\quad \var(z_k)=1 \ \text{and}\ z_k\in	\lspan(\bd{x}_k^\top)\setminus\lspan(\{z_k^{(m)}\}_{m=1}^{\ell-1}),
\end{split}
\ee
where $\lspan(\bd{x}_k^\top)\setminus\lspan(\{z_k^{(m)}\}_{m=1}^0):=\lspan(\bd{x}_k^\top)$,
and 
for $\ell>1$,
$\lspan(\bd{x}_k^\top)\setminus\lspan(\{z_k^{(m)}\}_{m=1}^{\ell-1})$
denotes
the 
orthogonal complement of $\lspan(\{z_k^{(m)}\}_{m=1}^{\ell-1})$ in $\lspan(\bd{x}_k^\top)$.
The correlation $\rho^{(\ell)}:=\corr(z_1^{(\ell)},z_2^{(\ell)})$  is called the $\ell$-th {\it canonical correlation} of $\bd{x}_1$ and $\bd{x}_2$.
Augment 
$\{z_k^{(\ell)}\}_{\ell=1}^{r_c}$ with any $(r_k-r_c)$ standardized real random variables to be $\bd{z}_k=[z_k^{(\ell)}]_{\ell=1}^{r_k}$
such that its entries form an orthonormal basis of $\lspan(\bd{x}_k^\top)$.
We have the bi-orthogonality \citep{Shu18}: 
\be\label{bi-orthogonality}
\cov(\bd{z}_1,\bd{z}_2)=
\begin{bmatrix}
\diag(\rho^{(1)},\dots,\rho^{(r_c)})&\mb{0}_{r_c\times (r_2-r_c)}\\ \mb{0}_{(r_1-r_c)\times r_c}&\mb{0}_{(r_1-r_c)\times (r_2-r_c)}
\end{bmatrix}.
\ee
The augmented canonical variables (ACVs) $\{\bd{z}_k\}_{k=1}^2$ can be obtained by 
$
\bd{z}_k=\mb{U}_{\theta, k}^\top \bd{h}_k,
$
where $\bd{h}_k=\mb{\Lambda}_k^{-1/2}\mb{V}_k^\top \bd{x}_k$,
$\mb{V}_k\mb{\Lambda}_k\mb{V}_k^\top$ is the compact 
SVD of $\cov(\bd{x}_k)$,
and  $ \mb{U}_{\theta, 1} \mb{\Lambda}_{\theta} \mb{U}_{\theta, 2}^\top
$ is the full SVD of $\mb{\Theta}:=\cov(\bd{h}_1,\bd{h}_2)$ with $\mb{\Lambda}_{\theta}=\cov(\bd{z}_1,\bd{z}_2)$ given in \eqref{bi-orthogonality}.

\section{Two-view D-CDLF ($K=2$)}\label{sec: 2-view D-CDLF}
We begin with the decomposition of two standardized real random variables, and then extend it to any two real random vectors.
Following this, we introduce our proposed variable-level and view-level PVEs.

\subsection{Decomposition of two standardized real random variables}
Let $z_1$ and $z_2$ be two standardized real random variables with correlation $\rho\in [0,1]$. We aim to decompose them by
\be\label{decomp simple case}
z_k=c+d_k~~\text{for}~~k=1,2,
\ee
with a common variable $c$ and two distinctive variables $d_1$ and $d_2$ in $(\mathcal{L}_0^2,\cov)$ subject to 
\be\label{tri-orthogonal}
c \perp \{d_1,d_2\}~~\text{and}~~d_1\perp d_2.
\ee
When $\rho\in (0,1)$, the tri-orthogonality constraint \eqref{tri-orthogonal} implies that $\lspan(\{c, d_1,d_2\})$ is three-dimensional 
and larger than $\lspan(\{z_1,z_2\})$. We thus need to expand our perspective on the decomposition
from $\lspan(\{z_1,z_2\})$ to a slightly larger space 
allowing the tri-orthogonality, for example, $\lspan(\{z_1,z_2,z_{\Im}\})$, with an auxiliary variable $z_{\Im}$ that
can be any standardized real random variable 
satisfying $z_{\Im}\perp \{z_1,z_2\}$.
The solutions to the decomposition in $\lspan(\{z_1,z_2,z_{\Im}\})$
are given in the following proposition.

\begin{prop}\label{prop: uniqueness of c}
Let $z_1,z_2,z_{\Im}\in (\mathcal{L}_0^2,\cov)$ be three standardized random variables with $\corr(z_1,z_2)=\rho\in [0,1]$ and $z_{\Im}\perp \{z_1,z_2\}$.
 Decomposition \eqref{decomp simple case}
in $\lspan(\{z_1,z_2,z_{\Im}\})$ with constraint \eqref{tri-orthogonal} only has the 
solutions 
$
c
=(z_1+z_2)\frac{\rho}{1+\rho}\pm z_{\Im}\sqrt{\frac{\rho(1-\rho)}{1+\rho}}.
$
Moreover, $\var(c)=\rho$ and $\var(d_1)=\var(d_2)=1-\rho$.
\end{prop}	

Since $-z_{\Im}$ is also orthogonal to $z_1$ and $z_2$,
for simplicity we let 
 \be\label{c solution}
c
=(z_1+z_2)\frac{\rho}{1+\rho}+ z_{\Im}\sqrt{\frac{\rho(1-\rho)}{1+\rho}}
=:c_{\Re}+c_{\Im}.
\ee
We call $c_{\Re}$ the real part of $c$ and call $c_{\Im}$ the imaginary part 
of $c$.
Similarly, 
for $k\in\{1,2\}$,
$d_{\Re,k}=z_k-c_{\Re}$
and $d_{\Im,k}=-c_{\Im}$ are called
the real and imaginary parts of $d_k$, respectively.

From Proposition~\ref{prop: uniqueness of c},
the proportions of the variance of standardized variable $z_k$ explained by $c$ and $d_k$ are $\var(c)=\rho$ and $\var(d_k)=1-\rho$, respectively,
which well link the contribution roles of $c$ and $d_k$
in generating $z_1$ and $z_2$
with their correlation $\rho$.
In particular, 
$c=z_1=z_2$ and $d_1=d_2=0$ if $\rho=1$, and $c=0$ and $d_k=z_k$ if $\rho=0$.
Figure~\ref{f:D-CDLF for 2v} illustrates the decomposition changing with the correlation $\rho$.

\begin{figure}
	\centerline{\includegraphics[width=1\textwidth]{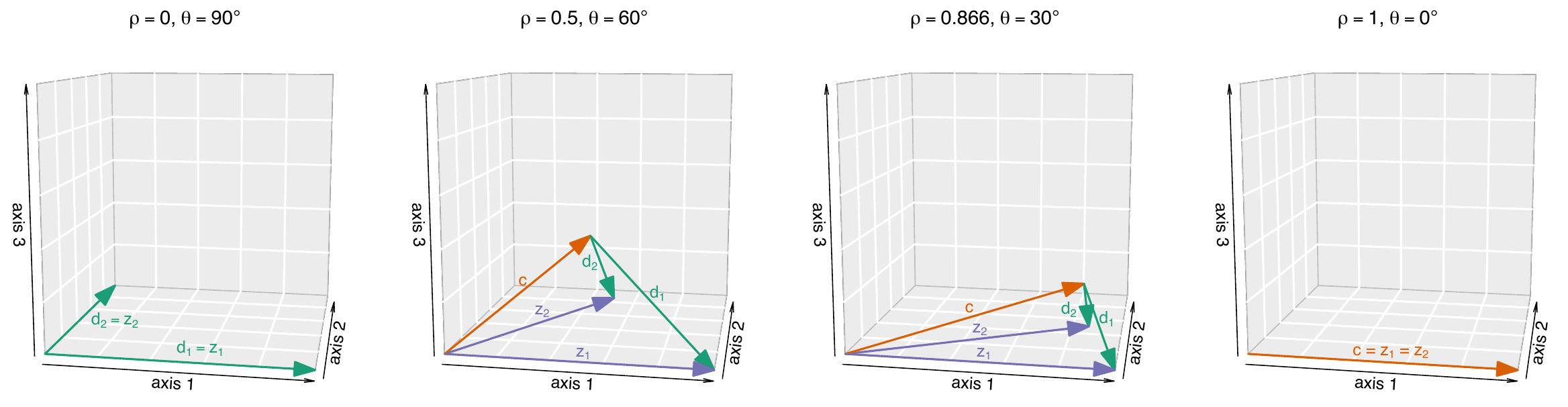}}
	\caption{The D-CDLF decomposition for two standardized real random variables $z_1$ and $z_2$ with different values of their correlation $\rho$ and angle $\theta=\arccos \rho$.}
	\label{f:D-CDLF for 2v}
\end{figure}

\begin{remark}
Since $z_{\Im}$ only serves as an auxiliary role to form the three-dimensional space, we may write $c$ in an equivalent way 
as a complex random variable 
with $z_{\Im}$ put on the imaginary part of $c$. That is,
$
c
=(z_1+z_2)\frac{\rho}{1+\rho}\pm \mathbf{i} z_{\Im}\sqrt{\frac{\rho(1-\rho)}{1+\rho}}.
$
Then the tri-orthogonality \eqref{tri-orthogonal} holds in the space $\lspan_{\mathbb{R}}(\{z_1,z_2,\mathbf{i}z_{\Im}\})$,
where orthogonality  is also equivalent to uncorrelatedness.
\end{remark}

\subsection{Decomposition of  two real random vectors}

We aim to extend the two-variable decomposition in previous subsection to any two real random vectors $\{\bd{x}_k\}_{k=1}^2$  in $(\mathcal{L}_0^2,\cov)$. 
Recall that the ACVs $(z_k^{(\ell)})_{\ell=1}^{r_k}=\bd{z}_k^\top $ of CCA given in Section~\ref{subsec: CCA} form an orthonormal basis of $\lspan(\bd{x}_k^\top)$.
We can thus write $\bd{x}_k$ as a linear combination of these ACVs and then apply the two-variable decomposition to each paired ACVs.
Specifically, we have
\be\label{tb: x=c+d}
\bd{x}_k=\sum_{\ell=1}^{r_k}\bd{\beta}_k^{(\ell)}z_k^{(\ell)}
=\sum_{\ell=1}^{r_k}\bd{\beta}_k^{(\ell)}(c^{(\ell)}+d_k^{(\ell)})
=\sum_{\ell=1}^{r_c}\bd{\beta}_k^{(\ell)}c^{(\ell)}+\sum_{\ell=1}^{r_k}\bd{\beta}_k^{(\ell)}d_k^{(\ell)},
\ee
where $z_k^{(\ell)}=c^{(\ell)}+d_k^{(\ell)}$,  $c^{(\ell)}=c_{\Re}^{(\ell)}+c_{\Im}^{(\ell)}$ is given in \eqref{c solution} with $\{z_1,z_2,z_{\Im},\rho\}$ replaced by $\{z_1^{(\ell)},z_2^{(\ell)},z_{\Im}^{(\ell)},\rho^{(\ell)}\}$ for $\ell\le r_1\wedge r_2$ and $c^{(\ell)}=0$ for $\ell>r_1\wedge r_2$,
and $\mb{B}_k:=(\bd{\beta}_k^{(\ell)})_{\ell=1}^{r_k}=\cov(\bd{x}_k,\bd{z}_k)$.
Note that $c^{(\ell)}=0$ for $r_c<\ell\le r_1\wedge r_2$.
For convenience, we set $\rho^{(\ell)}=0$ for $\ell>r_1\wedge r_2$.

From equation \eqref{tb: x=c+d}, 
we define $\{c^{(\ell)}\}_{\ell=1}^{r_c}$ as the CLFs of $\{\bd{x}_k\}_{k=1}^2$,
and
$\{d_k^{(\ell)}\}_{\ell=1}^{r_k}$ as the DLFs of~$\bd{x}_k$.
The common-source and distinctive-source random vectors, $\bd{c}_k$ and $\bd{d}_k$, of $\bd{x}_k$ are defined by 
\begin{align}
\bd{c}_k
&=\mb{B}_k^{[:,1:r_c]}[c^{(\ell)}]_{\ell=1}^{r_c}=
\mb{B}_k^{[:,1:r_c]}[c_{\Re}^{(\ell)} + c_{\Im}^{(\ell)}  ]_{\ell=1}^{r_c}
=:\bd{c}_{\Re,k}+\bd{c}_{\Im,k},\label{tb: two components for c_k}\\
\bd{d}_k&=\mb{B}_k[d_k^{(\ell)}]_{\ell=1}^{r_k}=\bd{x}_k-\bd{c}_k
=(\bd{x}_k-\bd{c}_{\Re,k})-\bd{c}_{\Im,k}
=:\bd{d}_{\Re,k}+\bd{d}_{\Im,k}.\label{tb: two components for d_k}
\end{align}
Let $z_{\Im}^{(\ell)}$ denote the auxiliary variable corresponding to $c^{(\ell)}$ in \eqref{c solution}.
We set $\{z_{\Im}^{(\ell)}\}_{\ell=1}^{r_c}\perp
\{z_k^{(\ell)}\}_{\ell\in[r_k],k\in[2]}$ and
$z_{\Im}^{(\ell_1)} \perp z_{\Im}^{(\ell_2)}$ for $\ell_1\ne \ell_2$.
Then by the bi-orthogonality of ACVs in \eqref{bi-orthogonality},
we obtain
the orthogonality of CLFs and DLFs:
$c^{(\ell_1)}\perp c^{(\ell_2)}$ and $d_{k}^{(\ell_1)}\perp d_{k}^{(\ell_2)}$ for $\ell_1\ne \ell_2$,
$d_1^{(\ell_1')}\perp d_2^{(\ell_2')}$ for $\ell_k'\in[r_k]$, and
$\{c^{(\ell)}\}_{\ell=1}^{r_c}\perp 
\{d_k^{(\ell)}\}_{\ell\in[r_k],k\in[2]}$.
In other words, we have the following desirable orthogonality:
\be\label{2-view prop}
\begin{cases}
\lspan(\bd{c}_1^\top)=\lspan(\bd{c}_2^\top)=\lspan(\{c^{(\ell)}\}_{\ell=1}^{r_c}),
\\
\lspan(\bd{d}_1^\top)\perp \lspan(\bd{d}_2^\top)~\text{with}~ \lspan(\bd{d}_k^\top)=\lspan(\{d_k^{(\ell)}\}_{\ell=1}^{r_k}),
\\
\lspan(\bd{c}_1^\top)
\perp\lspan([\bd{d}_1;\bd{d}_2]^\top).
\end{cases}
\ee
From Proposition~\ref{prop: uniqueness of c}, 
the covariance matrices of $\bd{c}_k$ and $\bd{d}_k$
can be computed by
\begin{align}
\cov(\bd{c}_k)&=\mb{B}_k^{[:,1:r_c]} \diag([\rho^{(\ell)}]_{\ell=1}^{r_c})
(\mb{B}_k^{[:,1:r_c]})^\top,\label{tb: cov(c_k)}\\
\cov(\bd{d}_k)&
=\mb{B}_k\diag([1-\rho^{(\ell)}]_{\ell=1}^{r_k})\mb{B}_k^\top=\cov(\bd{x}_k)-\cov(\bd{c}_k).
\label{tb: cov(d_k)}
\end{align}

\begin{theorem}[Uniqueness]\label{thm: Uniqueness of C for two blocks}
For $k\in\{1,2\}$,
$\cov(\bd{c}_k)$, $\cov(\bd{d}_k)$, $\bd{c}_{\Re,k}$ and $\bd{d}_{\Re,k}$ given in \eqref{tb: cov(c_k)}, \eqref{tb: cov(d_k)}, \eqref{tb: two components for c_k} and \eqref{tb: two components for d_k} are unique to 
$\bd{x}_k$, regardless of the non-uniqueness of 
$\{z_1^{(\ell)}\}_{\ell=1}^{r_1}\cup\{z_2^{(\ell)}\}_{\ell=1}^{r_2}\cup\{z_{\Im}^{(\ell)}\}_{\ell=1}^{r_c}$.
\end{theorem}

\begin{remark}	
	Although the non-uniqueness of auxiliary variables $\{z_{\Im}^{(\ell)}\}_{\ell=1}^{r_c}$
	causes the non-identifiability issue of the CLF and DLF spaces
	$\{\lspan(\bd{c}_k^\top),\lspan(\bd{d}_k^\top)\}_{k=1}^2$,
	the covariance of $\bd{x}_k$ explained by the CLFs and DLFs, i.e.,
$\cov(\bd{c}_k)$ and $\cov(\bd{d}_k)$, are invariant as shown in Theorem~\ref{thm: Uniqueness of C for two blocks}. 
Moreover, 
	 to build a predictive model
	 for a real-valued outcome random vector $\bd{y}$  using
	 $E(\bd{y}|\{c^{(\ell)}\}_{\ell=1}^{r_c},\{d_k^{(\ell)}\}_{\ell\in [r_k],k\in[2]},V,\{z_{\Im}^{(\ell)}\}_{\ell=1}^{r_c})$,
	where $V$ is a set of real random variables (which can be  empty),
	 if 
	 we choose
$[z_{\Im}^{(\ell)}]_{\ell=1}^{r_c}$
to be 
independent of
	$[\bd{y};[c_{\Re}^{(\ell)}]_{\ell=1}^{r_c};
	[[d_{\Re,k}^{(\ell)}]_{\ell\in [r_k]}]_{k\in[2]};
	[V]]$,
	then 
$E(\bd{y}|\{c^{(\ell)}\}_{\ell=1}^{r_c},\{d_k^{(\ell)}\}_{\ell\in [r_k],k\in[2]},V,\{z_{\Im}^{(\ell)}\}_{\ell=1}^{r_c})
=E(\bd{y}|\{c_{\Re}^{(\ell)}\}_{\ell=1}^{r_c},\{d_{\Re,k}^{(\ell)}\}_{\ell\in [r_k],k\in[2]},V)$,  
	 and thus 
	 this predictive model is invariant to the non-uniqueness of $\{z_{\Im}^{(\ell)}\}_{\ell=1}^{r_c}$.

\end{remark}

\begin{remark}
Our D-CDLF only differs from D-CCA \citep{Shu18} in the CLFs
$\{c^{(\ell)}\}_{\ell=1}^{r_c}$.
D-CCA defines its $\ell$-th CLF as
$c^{(\ell)}=(z_1^{(\ell)}+z_2^{(\ell)})/\big(2-2\sqrt{(1-\rho^{(\ell)})/(1+\rho^{(\ell)})}\big)$,
which satisfies $\big(1+\sqrt{1-(\rho^{(\ell)})^2}\big)\var(c^{(\ell)})=(\rho^{(\ell)})^2\le \rho^{(\ell)}$.
Since 
D-CDLF has $\var(c^{(\ell)})=\rho^{(\ell)}$
and
both methods have 
$
\var(\bd{c}_k^{[i]})=\sum_{\ell=1}^{r_c}(\mb{B}_k^{[i,\ell]})^2\var(c^{(\ell)})
$, it concludes that
the variance
$\var(\bd{c}_k^{[i]})$ of D-CCA
is no larger than 
that of D-CDLF.
\end{remark}

\subsection{Variable-level and view-level  PVEs}
To measure the joint effect of CLFs  or DLFs on original variables, 
we propose the variable-level PVEs and the view-level
PVEs.

The variable-level PVEs for a denoised original variable $\bd{x}_k^{[i]}$
by CLFs
$\{c^{(\ell)}\}_{\ell=1}^{r_c}$  
and its DLFs
$\{d_k^{(\ell)}\}_{\ell=1}^{r_k}$ 
are, respectively,
defined as
\vspace{-0.08cm}
$$
\PVE_c(\bd{x}_k^{[i]}):=
\frac{\var(\bd{c}_k^{[i]})}{\var(\bd{x}_k^{[i]})}
=\sum_{\ell=1}^{r_c} \corr^2(\bd{x}_k^{[i]},c^{(\ell)}),
$$
and
$$
\PVE_d(\bd{x}_k^{[i]})
:=
\frac{\var(\bd{d}_k^{[i]})}{\var(\bd{x}_k^{[i]})}
=\sum_{\ell=1}^{r_k} \corr^2(\bd{x}_k^{[i]},d_k^{(\ell)}),
$$
which are equal to the sums of their squared correlations.
The variable-level PVEs are useful in selecting original  variables within each data view that are highly affected by CLFs and DLFs, respectively.

The view-level PVEs 
for the entire $\bd{x}_k$ by CLFs
$\{c^{(\ell)}\}_{\ell=1}^{r_c}$  
and its DLFs
$\{d_k^{(\ell)}\}_{\ell=1}^{r_k}$ 
are, respectively, defined~as
$$
\PVE_c(\bd{x}_k)
:=\frac{\sum_{i=1}^{p_k}\var(\bd{c}_k^{[i]})}{\sum_{i=1}^{p_k}\var(\bd{x}_k^{[i]})}
=\sum_{i=1}^{p_k}\gamma_{ki}\PVE_c(\bd{x}_k^{[i]})
$$
and
$$
\PVE_d(\bd{x}_k)
:=\frac{\sum_{i=1}^{p_k}\var(\bd{d}_k^{[i]})}{\sum_{i=1}^{p_k}\var(\bd{x}_k^{[i]})}
=\sum_{i=1}^{p_k}\gamma_{ki}\PVE_d(\bd{x}_k^{[i]}),
$$
which are the weighted averages of 
corresponding variable-level PVEs with weights
$\gamma_{ki}=\var(\bd{x}_k^{[i]})/\sum_{j=1}^{p_k}\var(\bd{x}_k^{[j]})$.

Due to the uncorrelatedness 
between CLFs and DLFs, the two types of PVEs follow {\it the rule of sum}:
$$
\PVE_c(\bd{x}_k^{[i]})+\PVE_d(\bd{x}_k^{[i]})=1
\qquad\text{and}\qquad
\PVE_c(\bd{x}_k)+\PVE_d(\bd{x}_k)=1.
$$

\section{Estimation}\label{sec: estimation}
\label{section: est for K>=2}
 Suppose that the high-dimensional low-rank plus noise structure in \eqref{eqn: y=c+d+e}
 follows the factor model \citep{Shu2022d}:
\be\label{factor model}
\mb{Y}_k=\mb{X}_k+\mb{E}_k=\mb{B}_{k,f}\mb{F}_k+\mb{E}_k,
\qquad
\bd{y}_k=\bd{x}_k+\bd{e}_k=\mb{B}_{k,f}\bd{f}_k+\bd{e}_k,
\qquad k=1,2,
\ee
where $\mb{B}_{k,f}\in \mathbb{R}^{p_k\times r_k}$ is a deterministic matrix, the columns of $\mb{Y}_k$, $\mb{X}_k$, $\mb{F}_k$ and $\mb{E}_k$ are  the $n$ i.i.d. copies of $\bd{y}_k$, $\bd{x}_k$, $\bd{f}_k$ and $\bd{e}_k$, respectively,
$\bd{f}_k^\top$ is an orthonormal 
basis of $\lspan(\bd{x}_k^\top)$ with $\cov(\bd{f}_k,\bd{e}_k)=\mb{0}_{r_k\times p_k}$,
$\lspan(\bd{x}_k^\top)$ is a fixed space that is independent of $\{p_k\}_{k=1}^2$ and $n$, {\color{black}and $\mb{F}:=[\mb{F}_1;\mb{F}_2]$ has independent columns}.
We assume
that $\cov(\bd{y}_k)$ is a spiked covariance matrix,  for which
the largest $r_k$ eigenvalues are significantly larger than the rest, i.e., signals are distinguishably stronger than noises.
The $r_k$ spiked eigenvalues are majorly contributed by signal $\bd{x}_k$, 
whereas the rest small eigenvalues are induced by noise $\bd{e}_k$.
The spiked covariance model has been widely used in various fields, such as signal processing \citep{Nada10}, machine learning \citep{Huang17}, and economics \citep{Cham83}.

For simplicity, we define D-CDLF estimators 
using true $\{r_k\}_{k=1}^2$ and $r_c$, which can be estimated by
the edge distribution (ED) method of \citet{Onat10} and the minimum description length information-theoretic criterion (MDL-IC) of 
\citet{Song16}, respectively. See Section 2.3 in \citet{Shu18} for details.

The estimator of $\mb{X}_k$ is defined by using 
the soft-thresholding method of \citet{Shu18} as
\be\label{X tilde}
\widehat{\mb{X}}_k=\mb{U}_{k1} \diag([\widehat{\sigma}_\ell^S(\mb{Y}_k)]_{\ell=1}^{r_k})\mb{U}_{k2}^\top,
\ee
where
$\mb{U}_{k1}\diag([\sigma_\ell(\mb{Y}_k)]_{\ell=1}^{r_k})\mb{U}_{k2}^\top$ is the top-$r_k$ SVD of $\mb{Y}_k$,
and the soft-thresholded singular value
$
\widehat{\sigma}_\ell^S(\mb{Y}_k)=\sqrt{\max\{\sigma_\ell^2(\mb{Y}_k)- \tau_kp_k,0\}}
$ with
$
\tau_k=\sum_{\ell=r_k+1}^{p_k} \sigma_\ell^2(\mb{Y}_k)/(np_k-nr_k-p_kr_k).
$ 
Define the estimator of $\mb{\cov}(\bd{x}_k)$ by
\be\label{est cov(x_k)}
\widehat{\mb{\cov}}(\bd{x}_k)=n^{-1}\widehat{\mb{X}}_k\widehat{\mb{X}}_k^\top,
\ee
and denote its top-$r_k$ SVD by 
$\widehat{\mb{\Sigma}}_k=\widehat{\mb{V}}_k\widehat{\mb{\Lambda}}_k\widehat{\mb{V}}_k^\top$,
where $\widehat{\mb{\Lambda}}_k=\diag([\sigma_\ell(\widehat{\cov}(\bd{x}_k)]_{\ell=1}^{r_k})$.
Let $\widehat{\mb{H}}_k=(\widehat{\mb{\Lambda}}_k^\dag)^{1/2}\widehat{\mb{V}}_k^\top \widehat{\mb{X}}_k$, which is the estimated sample matrix of $\bd{h}_k =\mb{\Lambda}_k^{-1/2}\mb{V}_k^\top \bd{x}_k$.
Define the estimator of $\mb{\Theta}=\cov(\bd{h}_1,\bd{h}_2)$
by $\widehat{\mb{\Theta}}=n^{-1}\widehat{\mb{H}}_1\widehat{\mb{H}}_2^\top$.
Write
the full SVD of $\widehat{\mb{\Theta}}$ by
$
\widehat{\mb{\Theta}}=\widehat{\mb{U}}_{\theta, 1}\widehat{\mb{\Lambda}}_\theta \widehat{\mb{U}}_{\theta, 2}^\top
$.
The sample matrix of $\bd{z}_k$, the vector consisting of $\bd{x}_k$'s ACVs, is estimated by 
$\widehat{\mb{Z}}_k=\widehat{\mb{U}}_{\theta, k}^\top\widehat{\mb{H}}_k$.
We define the estimators of the canonical correlation
$\rho^{(\ell)}=\sigma_\ell(\mb{\Theta})$
and the coefficient matrix $\mb{B}_k=\cov(\bd{x}_k,\bd{z}_k)=\mb{V}_k\mb{\Lambda}_k^{1/2}\mb{U}_{\theta, k}$, respectively,
by 
\be\label{rho_hat and B_k_hat}
\widehat{\rho}^{(\ell)}=\sigma_\ell(\widehat{\mb{\Theta}})\quad\text{and}\quad
\widehat{\mb{B}}_k=n^{-1}\widehat{\mb{X}}_k\widehat{\mb{Z}}_k^\top=\widehat{\mb{V}}_k\widehat{\mb{\Lambda}}_k^{1/2}\widehat{\mb{U}}_{\theta, k}.
\ee

Then from the expressions of $\cov(\bd{c}_k)$ and $\cov(\bd{d}_k)$ given in \eqref{tb: cov(c_k)} and \eqref{tb: cov(d_k)}, their estimators are defined as
\begin{align}
\widehat{\cov}(\bd{c}_k)&=\widehat{\mb{B}}_k^{[:,1:r_c]} \diag([\widehat{\rho}_\ell]_{\ell=1}^{r_c})
(\widehat{\mb{B}}_k^{[:,1:r_c]})^\top,\label{est cov(c_k)}\\
\widehat{\cov}(\bd{d}_k)&
=\widehat{\cov}(\bd{x}_k)-\widehat{\cov}(\bd{c}_k).
\label{est cov(d_k)}
\end{align}
The proportions of signal variance explained by
CLFs and DLFs are estimated by
\begin{align*}
\widehat{\PVE}_c(\bd{x}_k)&=\frac{\tr(\widehat{\cov}(\bd{c}_k))}{\tr(\widehat{\cov}(\bd{x}_k))}=1-\frac{\tr(\widehat{\cov}(\bd{d}_k))}{\tr(\widehat{\cov}(\bd{x}_k))}=1-\widehat{\PVE}_d(\bd{x}_k),
\\
\widehat{\PVE}_c(\bd{x}_k^{[i]})&=\frac{\widehat{\var}(\bd{c}_k^{[i]})}{\widehat{\var}(\bd{x}_k^{[i]})}
=1-\frac{\widehat{\var}(\bd{d}_k^{[i]})}{\widehat{\var}(\bd{x}_k^{[i]})}
=1-\widehat{\PVE}_d(\bd{x}_k^{[i]}),
\end{align*}
where $\widehat{\var}(\bd{v}_k^{[i]})=\widehat{\cov}(\bd{v}_k)^{[i,i]}$ for $\bd{v}_k\in \{\bd{x}_k,\bd{c}_k,\bd{d}_k\}$.


To estimate the common-source and distinctive-source matrices
$\{\mb{C}_k,\mb{D}_k\}$, i.e., the sample matrices
of $\{\bd{c}_k,\bd{d}_k\}$,
we first need to find a centered isotropic random vector $\bd{z}_{\Im}
=[z_{\Im}^{(\ell)}]_{\ell=1}^{r_c}$ 
such that 
$\lspan(\bd{z}_{\Im}^\top)\perp\lspan([\bd{x}_k]_{k\in[K]}^\top)$, and its sample matrix $\mb{Z}_{\Im}$.
We generate $\{z_{\Im}^{(\ell)}\}_{\ell=1}^{r_c}$ as i.i.d. standard Gaussian random variables 
independent of $\{\bd{x}_k\}_{k=1}^2$ by using a Gaussian random number generator. 
Following the definitions of $\{\bd{c}_k,\bd{d}_k\}$ in
 \eqref{tb: two components for c_k} and \eqref{tb: two components for d_k}, we define the estimators of their sample matrices corresponding to $\{\mb{X}_1,\mb{X}_2\}$ by
  \be\label{C_k for two block}
\widehat{\mb{C}}_k= 
\widehat{\mb{B}}_k^{[:,1:r_c]}
[ \widehat{\bd{c}}^{(\ell)}]_{\ell=1}^{r_c}
\quad\text{and}\quad
\widehat{\mb{D}}_k=
\widehat{\mb{B}}_k[\widehat{\bd{d}}_k^{(\ell)}]_{\ell=1}^{r_k}=
\widehat{\mb{X}}_k-\widehat{\mb{C}}_k,
\ee
where  $\widehat{\bd{c}}^{(\ell)}$, $\widehat{\bd{d}}_k^{(\ell)}\in\mathbb{R}^{1\times n}$
are estimated samples of $c^{(\ell)}$ and $d_k^{(\ell)}$ defined by
 \be\label{samples of c(ell) and d(ell)}
 \widehat{\bd{c}}^{(\ell)}=\frac{\widehat{\rho}^{(\ell)}}{1+\widehat{\rho}^{(\ell)}}(\widehat{\mb{Z}}_1^{[\ell,:]}+\widehat{\mb{Z}}_2^{[\ell,:]})+ \mb{Z}_{\Im}^{[\ell,:]}\sqrt{\frac{\widehat{\rho}^{(\ell)}(1-\widehat{\rho}^{(\ell)})}{1+\widehat{\rho}^{(\ell)}}}
\quad\text{and}\quad
 \widehat{\bd{d}}_k^{(\ell)}=\widehat{\mb{Z}}_k^{[\ell,:]}- \widehat{\bd{c}}^{(\ell)}
 \ee
 for $\ell\le r_c$,
 and $\widehat{\bd{d}}_k^{(\ell)}=\widehat{\mb{Z}}_k^{[\ell,:]}$ 
for $\ell> r_c$ due to $\rho^{(\ell)}=0$,
 and $\widehat{\mb{Z}}_k$ and $\{\widehat{\rho}^{(\ell)}, \widehat{\mb{B}}_k\}$
are given above and in \eqref{rho_hat and B_k_hat}, respectively.

\section{Theoretical Proofs}\label{sec: proofs}
\begin{proof}[Proof of Proposition~\ref{prop: uniqueness of c}]
	Let $z_{\Im}\perp \{z_1,z_2\}$,  and    $w=a_1 z_1+a_2z_2+a_{\Im}z_{\Im}$ with $\var(w)=a_1^2+a_2^2+2a_1a_2\rho+a_{\Im}^2=1$. 
	Let $c=\alpha w$. We have
	\begin{align}
	\cov(c,d_1)=\cov(c,z_1-c)&=\cov(c,z_1)-\alpha^2\nonumber\\
	&=\alpha(a_1+a_2\rho)-\alpha^2=0\label{cov(c,d_1)=0}
	\end{align}
	and
	\begin{align}
	\cov(c,d_2)=\cov(c,z_2-c)&=\cov(c,z_2)-\alpha^2\nonumber\\
	&=\alpha(a_2+a_1\rho)-\alpha^2=0.\label{cov(c,d_2)=0}
	\end{align}
	So, $\sum_{k=1}^2\cov(z_k,c)=2\alpha^2$.
	Consequently,
	\[
	\cov(d_1,d_2)=\cov(z_1-c,z_2-c)=\cov(z_1,z_2)-\sum_{k=1}^2\cov(z_k,c)+\alpha^2
	=\rho-\alpha^2=0
	\]
	yields 
	\be\label{alpha value}
	\alpha=\pm\sqrt{\rho}.
	\ee
	Also from \eqref{cov(c,d_1)=0} and \eqref{cov(c,d_2)=0}, we obtain
	\be
	\alpha(1+\rho)(a_1+a_2)=2\alpha^2,
	\ee
	\be
	\alpha(a_1+a_2\rho)=\alpha^2,
	\ee
	and
	\be\label{alpha(a2+a1rho)}
	\alpha(a_2+a_1\rho)=\alpha^2.
	\ee
	
	When $\rho=0$, then by \eqref{alpha value} we have $\alpha=0$ and thus $c=\alpha w=0$.

	We next consider $\rho\in (0,1]$.
	Then from equations \eqref{alpha value}--\eqref{alpha(a2+a1rho)}, we
	have
	\be\label{a_1+a_2 value}
	a_1+a_2=2\alpha/(1+\rho)
	\ee
	and
	\[
	a_1+a_2\rho=a_2+a_1\rho=\alpha.
	\]
	Hence, $1-a_{\Im}^2=a_1^2+2a_1a_2\rho+a_2^2=\alpha(a_1+a_2)=2\rho/(1+\rho)$.
	So, 
	\be\label{a_{Im} value}
	a_{\Im}=\pm\sqrt{\frac{1-\rho}{1+\rho}}
	\ \text{for}\ \rho\in (0,1].
	\ee
	Moreover, 
	$(a_1+a_2)^2=4\rho/(1+\rho)^2=1-2a_1a_2\rho-a_{\Im}^2+2a_1a_2$,
	and thus $2a_1a_2(1-\rho)=4\rho/(1+\rho)^2-1+a_{\Im}^2=4\rho/(1+\rho)^2-2\rho/(1+\rho)=2\rho(1-\rho)/(1+\rho)^2$.
	Further let $\rho\in(0,1)$,
	then $a_1a_2=\rho/(1+\rho)^2$.
	Consequently, $(a_1-a_2)^2=(a_1+a_2)^2-4a_1a_2=4\rho/(1+\rho)^2-4\rho/(1+\rho)^2=0$. Then from \eqref{a_1+a_2 value}, $a_1=a_2=\alpha/(1+\rho).$
	We hence obtain
	\be\label{app: c solution}
	c=\alpha w=\alpha (a_1z_1+a_2z_2+a_{\Im}z_{\Im})
	=\frac{\rho}{1+\rho}(z_1+z_2)\pm z_{\Im}\sqrt{\frac{\rho(1-\rho)}{1+\rho}}
	\ee
	for $\rho\in (0,1)$.
	
	When $\rho=1$, we have $z_1=z_2$, and moreover,
	by \eqref{a_{Im} value} we have $a_{\Im}=0$.
	Then by \eqref{a_1+a_2 value} and \eqref{alpha value}, $c=\alpha w=\alpha(a_1z_1+a_2z_2+a_{\Im}z_{\Im})=\alpha(a_1+a_2)z_1=2\alpha^2z_1/(1+\rho)=2\rho z_1/(1+\rho)=(z_1+z_2)\rho/(1+\rho)$.
	The equation $c=(z_1+z_2)\rho/(1+\rho)$ satisfies \eqref{app: c solution} for $\rho=1$.
	
	From the above, for $\rho\in[0,1]$ we have 
	\[
	c
	=\frac{\rho}{1+\rho}(z_1+z_2)\pm z_{\Im}\sqrt{\frac{\rho(1-\rho)}{1+\rho}}
	\]
	and 
	$\var(c)=\alpha^2=\rho$.
	By $c\perp d_k$ for $k=1,2$, we have $\var(d_k)=\var(z_k)-\var(c)=1-\rho$.
\end{proof}

\begin{proof}[Proof of Theorem~\ref{thm: Uniqueness of C for two blocks}]
	Let $\{\widetilde{\bd{z}}_k^{[1:r_c]}\}_{k=1}^2$ be an arbitrary set of the first $r_c$ pairs of canonical variables of $\{\bd{x}_k\}_{k=1}^2$.
	From the proof of Theorem 2 in \citet{Shu18}, we have
	that there exists a matrix $\mb{Q}_{\theta}=\diag(\{\mb{Q}_{\theta\ell}\}_{\ell=1}^{r_c^*})$
	such that $\widetilde{\bd{z}}_k^{[1:r_c]}=\mb{Q}_\theta \bd{z}_k^{[1:r_c]}$,
	where $\mb{Q}_{\theta \ell}$ is an orthogonal matrix with column dimension equal to the repetition number of the $\ell$-th largest distinct value in $\rho_1,\dots, \rho_{r_c}$.
Then,
\begin{align}
\lefteqn{\cov(\bd{x}_k,\widetilde{\bd{z}}_k^{[1:r_c]}) \diag(\rho_1,\dots, \rho_{r_c})
	\cov(\widetilde{\bd{z}}_k^{[1:r_c]},\bd{x}_k)}
\label{prf tb: uniq of cov(c_k) part1}\\
	&=\cov(\bd{x}_k,\bd{z}_k^{[1:r_c]}) \mb{Q}_\theta^\top\diag(\rho_1,\dots, \rho_{r_c})
\mb{Q}_\theta	\cov(\bd{z}_k^{[1:r_c]},\bd{x}_k)
\nonumber\\
&= \cov(\bd{x}_k,\bd{z}_k^{[1:r_c]}) \diag(\rho_1,\dots, \rho_{r_c})
\cov(\bd{z}_k^{[1:r_c]},\bd{x}_k)
\label{prf tb: uniq of cov(c_k) part2}\\
&= \cov(\bd{x}_k,\bd{z}_k^{[1:r_c]}) \cov((c_1,\dots, c_{r_c})^\top)
\cov(\bd{z}_k^{[1:r_c]},\bd{x}_k)
\nonumber\\
&=\cov(\cov(\bd{x}_k,\bd{z}_k^{[1:r_c]})(c_1,\dots, c_{r_c})^\top)
\nonumber\\
&=\cov(\bd{c}_k).\nonumber
\end{align}	
By \eqref{prf tb: uniq of cov(c_k) part1} and \eqref{prf tb: uniq of cov(c_k) part2}, 
$
\cov(\bd{c}_k)=\cov(\bd{x}_k,\bd{z}_k^{[1:r_c]}) \diag(\rho_1,\dots, \rho_{r_c})
\cov(\bd{z}_k^{[1:r_c]},\bd{x}_k)
$
is unique.

Since $\lspan(\bd{c}_k^\top)\perp \lspan(\bd{d}_k^\top)$, that is, $\cov(\bd{c}_k,\bd{d}_k)=\mb{0}_{p_k\times p_k}$,
we have that $\cov(\bd{x}_k)=\cov(\bd{c}_k+\bd{d}_k)=\cov(\bd{c}_k)+\cov(\bd{d}_k)$.
Thus, $
\cov(\bd{d}_k)=\cov(\bd{x}_k)-\cov(\bd{c}_k)
$ is also unique,
and $\cov(\bd{d}_k)=\mb{B}_k\mb{B}_k^\top-
\mb{B}_k  \diag([\rho_\ell]_{\ell=1}^{r_k})    \mb{B}_k^\top
=\mb{B}_k  \diag([1-\rho_\ell]_{\ell=1}^{r_k})    \mb{B}_k^\top$.

Following the same proof technique used for Theorem~2 in \citet{Shu18}, we obtain the uniqueness of $\bd{c}_{\Re,k}$ and $\bd{d}_{\Re,k}$.
\end{proof}

\bibliographystyle{asa}
\bibliography{refs}

\end{document}